\newif\ifieee
\ieeefalse

\ifieee
\documentclass[letterpaper, 10 pt, conference]{ieeeconf}  
\usepackage{etoolbox}
\newtoggle{ieee}
\toggletrue{ieee}
\newtoggle{public}
\toggletrue{public}
\newtoggle{arxiv}
\togglefalse{arxiv}
\IEEEoverridecommandlockouts
\else

\documentclass[12pt]{article}
\usepackage[margin=2cm]{geometry}
\usepackage{etoolbox}
\newtoggle{ieee}
\togglefalse{ieee}
\newtoggle{public}
\toggletrue{public}
\newtoggle{arxiv}
\toggletrue{arxiv}

\fi

\iftoggle{ieee}
{
\usepackage[                   
  backend=bibtex8,             
  bibencoding=utf8,            
  style=numeric,            
  natbib=true,                 
  hyperref=true,               
  backref=false,                
  isbn=false,                  
  url=false,                    
  doi=false,                    
  urldate=long,                
  firstinits=true,
  terseinits=false,
  maxnames=3,%
  minnames=1,%
  maxbibnames=25,%
  minbibnames=15,%
  maxcitenames=2,%
  mincitenames=1%
]{biblatex}
\usepackage{ifthen}

\makeatletter
\newcounter{IEEE@bibentries}
\renewcommand\IEEEtriggeratref[1]{%
  \renewbibmacro{finentry}{%
    \stepcounter{IEEE@bibentries}%
    \ifthenelse{\equal{\value{IEEE@bibentries}}{#1}}
    {\finentry\newline}
    {\finentry}%
  }%
}
\makeatother
}
{
\usepackage[                   
  backend=bibtex8,             
  bibencoding=utf8,            
  style=alphabetic,            
  natbib=true,                 
  hyperref=true,               
  backref=true,                
  isbn=false,                  
  url=true,                    
  doi=true,                    
  urldate=long,                
  firstinits=true,
  terseinits=false,
  maxnames=3,%
  minnames=1,%
  maxbibnames=25,%
  minbibnames=15,%
  maxcitenames=2,%
  mincitenames=1%
]{biblatex}
}

\renewbibmacro*{doi+eprint+url}{%
  \iftoggle{bbx:doi}
    {\printfield{doi}}
    {}%
  \newunit\newblock
  \iftoggle{bbx:eprint}
    {\iffieldundef{doi}{\iffieldundef{url}{\printfield{eprint}}{}}{}}
    {}%
  \newunit\newblock
  \iftoggle{bbx:url}
    {\iffieldundef{doi}{\printfield{url}}{}}
    {}}

\usepackage[
    hypertexnames,%
    citecolor=blue,%
    colorlinks=true,%
    linkcolor=red%
]{hyperref}
\usepackage[all]{hypcap}

\usepackage[font=scriptsize,caption=false]{subfig}
\iftoggle{ieee}{
}
{
\usepackage{amsthm}
\usepackage{enumitem}
\usepackage{lineno}
}
\usepackage{amssymb,amsmath}
\usepackage{dsfont} 
\usepackage[mathscr]{eucal}

\usepackage{verbatim}
\usepackage{color}
\usepackage{graphicx}

\usepackage{tikz}
\usetikzlibrary{arrows}
\usetikzlibrary{calc}

\usepackage{xcolor}
\definecolor{darkblue}{HTML}{0000AD}
\definecolor{darkgreen}{HTML}{008600}
\definecolor{darkred}{HTML}{8B0000}
\definecolor{darkgray}{HTML}{666666}
\definecolor{_mage}{HTML}{912830}
\definecolor{_cyan}{HTML}{31837a}
\definecolor{_purp}{HTML}{49425c}

\usepackage{url}
\usepackage{thm-restate}

\usepackage{environ}
\NewEnviron{killcontents}{}

\newcommand{\fig}[1]{{Fig.~\ref{fig:#1}}}
\newcommand{\sct}[1]{{Sec.~\ref{sec:#1}}}

\newcommand{\see}[1]{(see~#1)}
\newcommand{\set}[1]{\left\{ #1 \right\}}
\newcommand{\paren}[1]{\left( #1 \right)}
\newcommand{\brak}[1]{\left[ #1 \right]}
\newcommand{\abs}[1]{\left| #1 \right|}
\newcommand{\ip}[2]{\langle #1,#2 \rangle}

\newcommand{\mat}[2]{\brak{\begin{array}{#1} #2 \end{array}}}

\newcommand{\td}[1]{\widetilde{#1}}

\newcommand{\eqnn}[1]{\begin{equation}\begin{aligned} #1 \end{aligned}\end{equation}}

\newcommand{\into}{\rightarrow}
\newcommand{\goesto}{\rightarrow}

\newcommand{\tr}{\top}

\newcommand{\I}{I_d}
\newcommand{\II}{I_{2d}}

\newcommand{\R}{\mathbb{R}}

\newcommand{\N}{\mathbb{N}}

\newcommand{\e}{\mathscr}
\newcommand{\q}{q}
\newcommand{\dq}{\dot{q}}
\newcommand{\p}{p}
\newcommand{\w}{\omega}

\newcommand{\vphi}{\varphi}

\newcommand{\id}{\I}

\newcommand{\diag}{\operatorname{diag}}

\newtheorem{proposition}{Proposition}
\newtheorem{definition}{Definition}
\newtheorem{theorem}{Theorem}
\newtheorem{corollary}{Corollary}
\newtheorem{lemma}{Lemma}
\newtheorem{claim}{Claim}
\newtheorem{assumption}{Assumption}
\newtheorem{remark}{Remark}
\newtheorem{example}{Example}

\newcommand{\defna}[2]{\begin{definition}[#1] #2 \end{definition}}
\newcommand{\rem}[1]{\begin{remark} #1 \end{remark}}
\newcommand{\remna}[2]{\begin{remark}[#1] #2 \end{remark}}

\newcommand{\asmpna}[2]{\begin{assumption}[#1] #2 \end{assumption}}

\newcommand{\asmpdiff}{Assump.~\ref{asmp:diff} (differentiable vector field and reset map)}

\newcommand{\asmpdecoupled}{Assump.~\ref{asmp:decoupled} (limbs decoupled through body)}
\newcommand{\asmptrj}{Assump.~\ref{asmp:trj} (admissible trajectories)}

\newcommand{\defcontact}{Def.~\ref{def:contact} (contact modes)}

\newcommand{\defadact}{Def.~\ref{def:adact} (admissible constraint activation/deactivation)}

\newcommand{\defseq}{Def.~\ref{def:seq} (contact mode sequence)}

\newcommand{\thmdiff}{Thm.~\ref{thm:diff} (differentiability through intermittent contact)}

\title{\LARGE \bf Decoupled limbs yield differentiable trajectory outcomes through intermittent contact in locomotion and manipulation}

\author{
Andrew~M.~Pace%
\and Samuel~A.~Burden%
\thanks{Department of Electrical Engineering, University of Washington, Seattle, WA, USA ({\tt apace2,sburden@uw.edu}). 
This material is based upon work supported by 
the U. S. Army Research Laboratory and the U. S. Army Research Office under contract/grant number W911NF-16-1-0158.
}
}

\date{}

\begin{document}

\maketitle
\thispagestyle{empty}
\pagestyle{empty}


\begin{abstract}
When limbs are decoupled, we find that 
trajectory outcomes in mechanical systems subject to unilateral constraints
vary differentiably with respect to initial conditions,
even as the contact mode sequence varies.
\end{abstract}


\section{Introduction}
\label{sec:intro}

Locomotion with legs entails intermittent contact with terrain;
manipulation with digits entails intermittent contact with objects.
Since legged locomotion \emph{is} self--manipulation~\cite{JohnsonKoditschek2013, JohnsonBurden2016ijrr}, mathematical models for intermittent contact between limbs and environments apply equally well to both classes of behaviors.
Parsimonious models for the dynamics of intermittent contact are piecewise-defined, with transitions between contact modes summarized by abrupt changes in system velocities.
Such models are \emph{hybrid} dynamical systems whose state evolution is governed by continuous-time \emph{flow} (generated by a vector field) punctuated by discrete-time \emph{reset} (specified by a map).
Trajectory outcomes are the resulting state of the system after \emph{flowing} and undergoing necessary 
resets for a specified period of time.
Trajectory outcomes in hybrid systems generally vary discontinuously as the discrete mode sequence varies as in~\fig{ex} (\emph{left}).
The point of this paper is to provide sufficient conditions that ensure trajectories in mechanical systems subject to unilateral constraints vary (continuously and) differentiably through intermittent contact, even as the contact mode sequence varies as in~\fig{ex} (\emph{right}).
Since scalable algorithms for optimization~\cite{Polak1997} and learning~\cite{SuttonBarto1998} rely on differentiability, conditions ensuring existence 
of derivatives are of practical importance in robotic locomotion and manipulation.

\subsection{Organization}
\label{sec:org}
We begin in~\sct{mdl} by specifying the class of dynamical systems under consideration, namely, \emph{mechanical} systems subject to \emph{unilateral} constraints.
\sct{asmp} imposes conditions on the system dynamics and trajectories that enable us in 
\sct{diff} to report that trajectories vary differentiably with respect to initial conditions, even as the contact mode sequence varies.

\subsection{Relation to prior work}
\label{sec:prior}
The technical content in~\sct{mdl} and~\sct{asmp} appeared previously in the literature and is (more--or--less) well--known; we collate the results here to contextualize and streamline our contributions in~\sct{diff}.

\begin{figure}[ht]
\centering
{
\iftoggle{ieee}
{
\hfill%
{
\def\svgwidth{0.5\columnwidth} 
\resizebox{.45\columnwidth}{!}{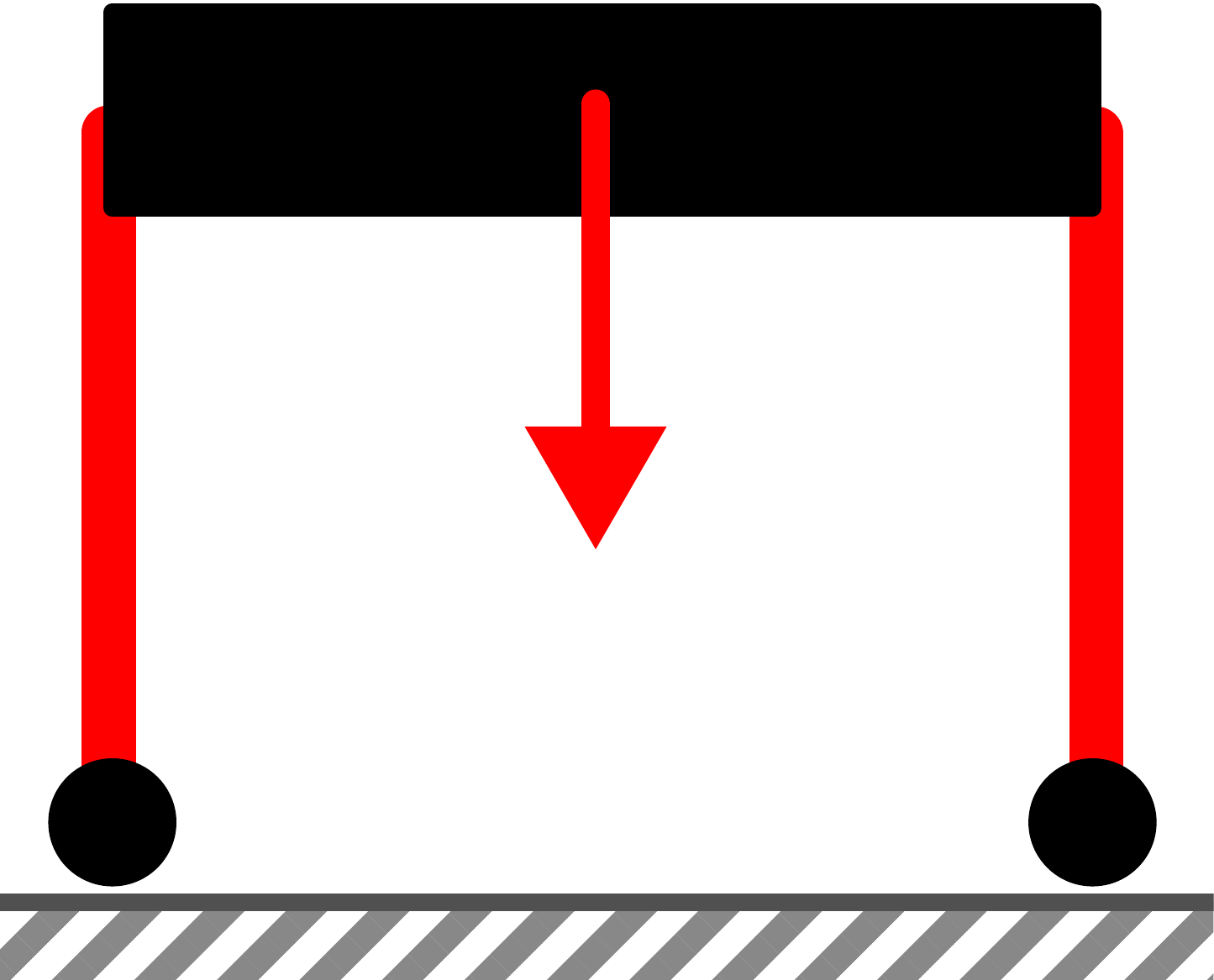}
}
\hfill%
{
\def\svgwidth{0.5\columnwidth} 
\resizebox{.45\columnwidth}{!}{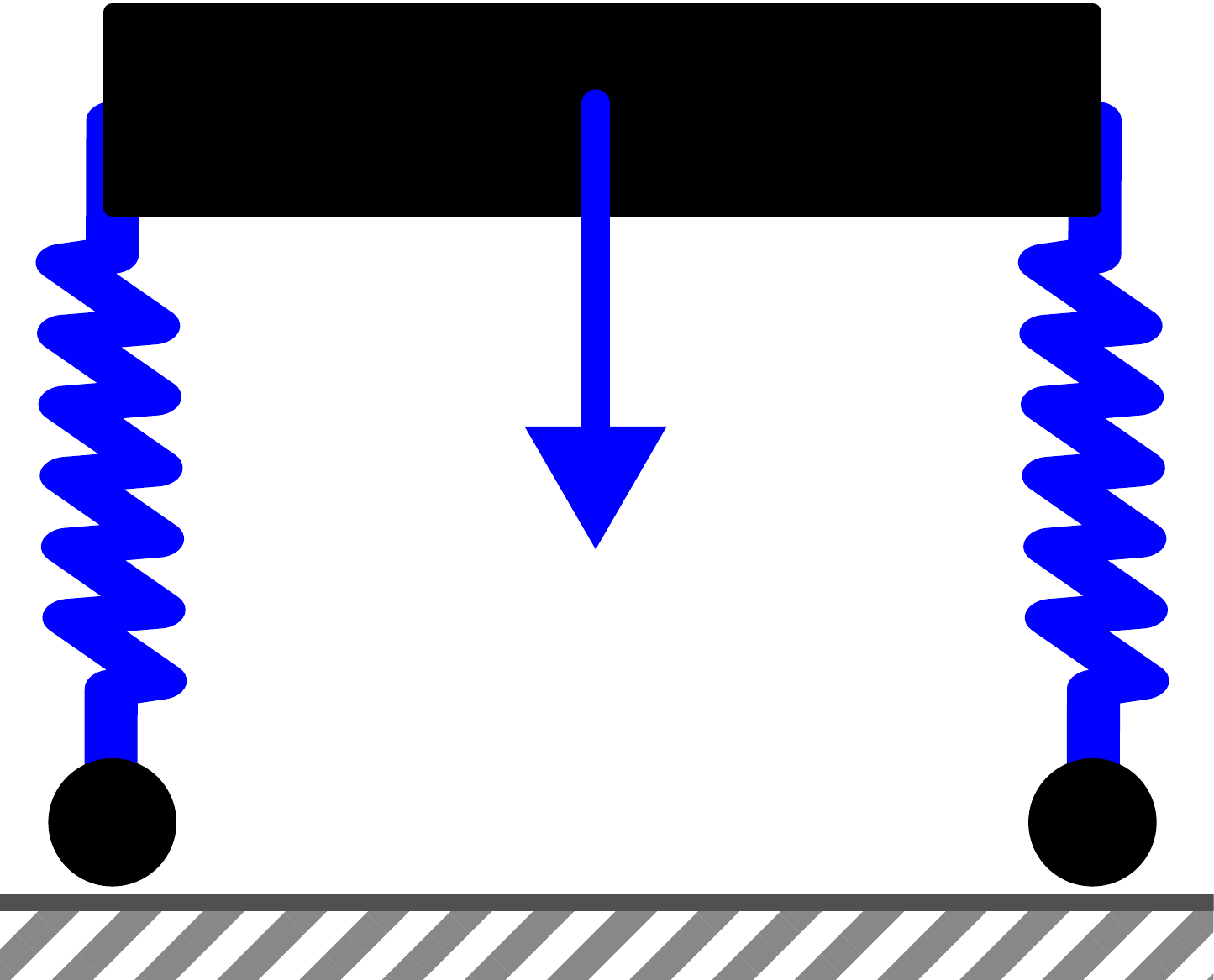}
}
\includegraphics[width=\columnwidth]{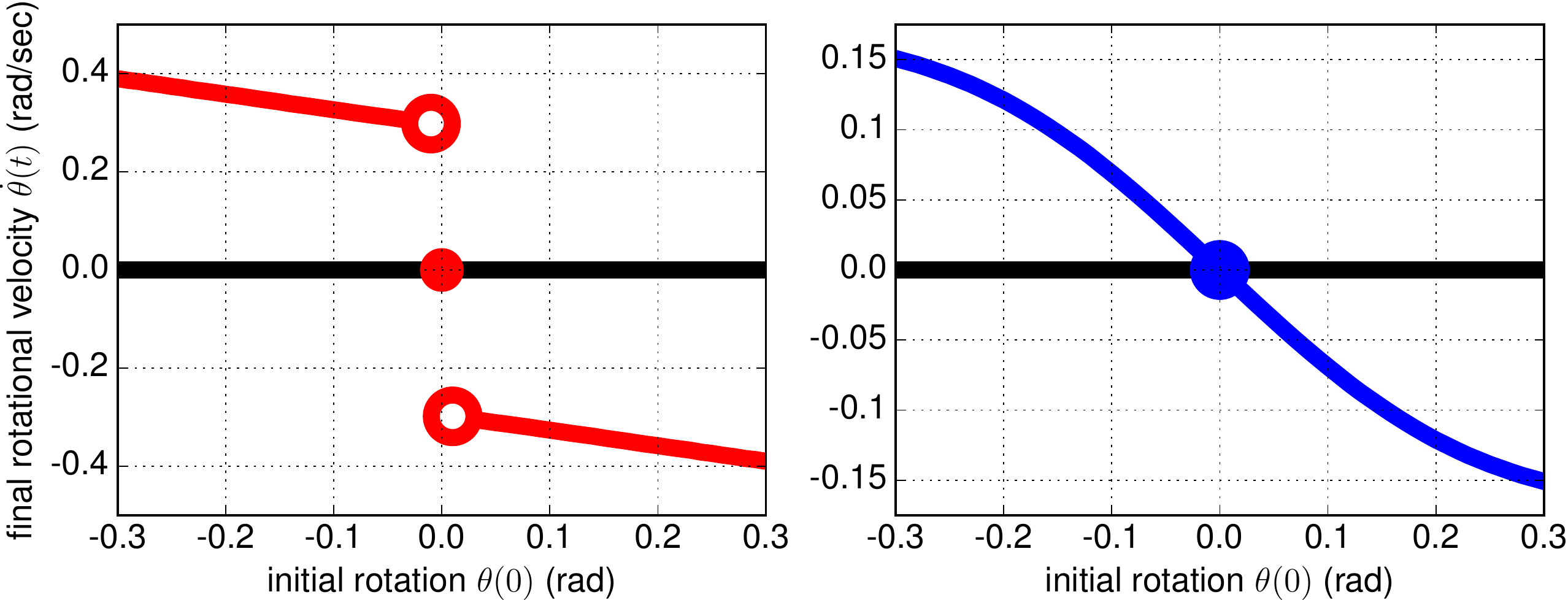}
}
{
\quad\includegraphics[width=.4\columnwidth]{trot_td_stiff.pdf}\quad\quad
\includegraphics[width=.4\columnwidth]{trot_td_soft.pdf}\\
\includegraphics[width=.9\columnwidth]{code/salt.pdf}
}
\caption{
\label{fig:ex}
Trajectory outcomes after flowing for a uniform time from the initial conditions away from impacts in 
mechanical systems subject to unilateral constraints.
(\emph{left}) 
In general, trajectory outcomes depend discontinuously on initial conditions.
In the pictured model for rigid--leg trotting (adapted from~\cite{RemyBuffinton2010}), discontinuities arise when two legs touch down:
if the legs impact simultaneously (corresponding to rotation $\theta(0) = 0$), then the post--impact rotational velocity is zero;
if the rear leg impacts before the front leg ($\theta(0) > 0$) or vice--versa ($\theta(0) < 0$), 
then the post--impact rotational velocities are bounded away from zero.
(\emph{right})
When limbs are decoupled (e.g. through viscoelasticity), trajectory outcomes depend continuously on initial conditions.
In the pictured model for soft--leg trotting (adapted from~\cite{BurdenGonzalezVasudevan2015tac}), 
trajectory outcomes (solid lines) are continuous and differentiable.
These figures were generated using simulations of the depicted models.
}
}
\end{figure}

\section{Mechanical systems subject to unilateral constraints}
\label{sec:mdl}

In this paper, we study the dynamics of a mechanical system with configuration coordinates $q\in Q=\R^d$ subject to 
unilateral constraints $a(q) \ge 0$ 
specified by a differentiable function %
$a : Q\into \R^n$
where $d,n\in\N$ are finite.
We are primarily interested in systems with $n > 1$ constraints,  
whence we regard the inequality $a(q) \ge 0$ as being enforced componentwise.
Given any $J\subset\set{1,\dots,n}$, 
and letting $\abs{J}$ denote the number of elements in the set $J$,
we let 
$a_J : Q \into \R^{\abs{J}}$ 
denote the function obtained by selecting the component functions of $a$ indexed by $J$, 
and we regard the equality $a_J(q) = 0$ as being enforced componentwise.
It is well--known~\see{e.g.~\cite[Sec.~3]{Ballard2000} or \cite[Sec.~2.4,~2.5]{JohnsonBurden2016ijrr}}
that with
$J = \set{j\in\set{1,\dots,n} : a_j(q) = 0}$ 
the system's dynamics take the form
\begin{subequations}\label{eq:dyn}
\begin{align}
  M(q)\ddot{q} & = f(q,\dot{q}) + c(q,\dot{q})\dot{q} + Da_J(q)^\tr \lambda_J(q,\dot{q}),\label{eq:dyn:cont}\\
  \dot{q}^+ & = \Delta_J(q,\dot{q}^-),\label{eq:dyn:disc}
\end{align}
\end{subequations}
where
$M : Q \into \R^{d\times d}$
specifies the mass matrix for the mechanical system in the $q$ coordinates,
$f : TQ \into \R^d$ 
is termed the \emph{effort map}~\cite{Ballard2000}
and specifies%
\footnote{We let $TQ = \R^d\times\R^d$ denote the \emph{tangent bundle} of the configuration space $Q$; an element $(q,\dot{q})\in TQ$ can be regarded as a pair containing a vector of generalized configurations $q\in\R^d$ and velocities $\dot{q}\in\R^d$; we write $\dot{q}\in T_q Q$.}
the internal and applied
forces, 
$c : TQ \into \R^{d\times d}$ 
denotes the \emph{Coriolis matrix} 
determined%
\footnote{For each $\ell,m\in\set{1,\dots,d}$
the $(\ell,m)$ entry $c_{\ell m}$ is determined 
from the entries of $M$ by the formula \\
$c_{\ell m} = -\frac{1}{2} \sum_{k=1}^d \paren{ D_k M_{\ell m} + D_m M_{\ell k} - D_\ell M_{k m}}$~\cite[Eqn.~30]{JohnsonKoditschek2013}.
}
by $M$,
$Da_J : Q \into \R^{\abs{J}\times d}$ 
denotes the (Jacobian) derivative of the constraint function $a_J$ with respect to the coordinates,
$\lambda_J : TQ \into \R^{\abs{J}}$ 
denotes the reaction forces generated in contact mode $J$ to enforce the constraint $a_J(q) \ge 0$,
\eqnn{
\lambda_J(q) = \paren{Da_J(q) M(q)^{-1} Da_J(q)^\tr}^{-1},
}
$\Delta_J : TQ \into \R^{d\times d}$ 
specifies the collision restitution law that instantaneously resets velocities to ensure compatibility with the constraint $a_J(q) = 0$,
\eqnn{
\dot{q}^+ = \Delta_J(q,\dot{q}^-) = \I - (1+\gamma(q,\dot{q}^-)) P_J(q)\dot{q}^-,
}
where 
$\I$ is the $(d\times d)$ identity matrix,
$\gamma : TQ\into[0,\infty)$ specifies the \emph{coefficient of restitution},
$P_J:Q\into \R^{d\times d}$ is the projection onto the constraint surface,
\eqnn{
P_J = M^{-1} Da_J^\tr \paren{Da_J M^{-1} Da_J^\tr}^{-1} Da_J,
}
and
$\dot{q}^+$ (resp. $\dot{q}^-$) denotes the right-- (resp. left--)handed limits of the velocity vector with respect to time.

\defna{contact modes}{
\label{def:contact}
The constraint functions $\set{a_j}_{j=1}^n$
partition the set of \emph{admissible} configurations 
$A = \set{q\in Q : a(q) \ge 0}$
into a finite collection%
\footnote{We let $2^n = \set{J \subset \set{1,\dots,n}}$ denote the \emph{power set} (i.e. the set containing all subsets) of $\set{1,\dots,n}$.}
$\set{A_J}_{J\in 2^n}$ 
of \emph{contact modes}:
\eqnn{
\forall J\in 2^n : A_J = \left\{q\in Q \mid \right. & a_J(q) = 0,\\ 
&\left. \forall i\not\in J : a_i(q) > 0\right\}.
}
For each $J\in 2^n$: 
we let $TA = \set{(q,\dot{q})\in TQ : q\in A}$ and $TA_J = \set{(q,\dot{q})\in TQ : q\in A_J}$;
if $q\in A_J$ then we say constraints in $J$ are \emph{active} at $q$.
}

\rem{
\label{rem:contact}
In~{\defcontact}, 
$J = \set{1,\dots,n}$ indexes the maximally constrained contact mode and 
$J = \emptyset$ indexes the unconstrained contact mode. 
}

\section{Assumptions}
\label{sec:asmp}
The point of this paper is to 
provide conditions that ensure trajectories of~\eqref{eq:dyn} vary differentiably as the contact mode sequence%
\footnote{See \defseq.}
varies.
Without imposing additional conditions,
the seemingly benign equations in~\eqref{eq:dyn} 
admit a range of dynamical phenomena that 
preclude differentiability.
This section contains the conditions that will enable us to obtain differentiable trajectory outcomes in~\sct{diff}.

\subsection{Existence and uniqueness of trajectories}

In the present paper, we will assume that appropriate conditions have been imposed to ensure trajectories of~\eqref{eq:dyn} exist on a region of interest in time and state.

\asmpna{existence and uniqueness}{
\label{asmp:flow}
There exists a \emph{flow} for~\eqref{eq:dyn},
that is,
a function $\phi:\e{F}\into TA$
where 
$\e{F}\subset[0,\infty)\times TA$ is an open subset (in the subspace topology)
containing $\set{0}\times TA$ 
and for each $(t,(q,\dot{q}))\in\e{F}$
the restriction
$\phi|_{[0,t]\times\set{(q,\dot{q})}}:[0,t]\into TQ$
is the unique left--continuous trajectory for~\eqref{eq:dyn}.
}

\rem{
The problem of ensuring trajectories of~\eqref{eq:dyn} exist and are unique has been studied extensively; 
we refer the reader to~\cite[Thm.~10]{Ballard2000} for a specific result, 
\cite[Thm.~5.3]{Brogliato2016} for a setup using constrained complementarity problems,
and~\cite{JohnsonBurden2016ijrr} 
for a general discussion of this problem.
}

\subsection{Differentiable vector field and reset map}
Since we are concerned with differentiability properties of the flow, 
we assume the elements in~\eqref{eq:dyn} are differentiable.

\asmpna{differentiable vector field and reset map}{
\label{asmp:diff}
The vector field~\eqref{eq:dyn:cont} and reset map~\eqref{eq:dyn:disc} are continuously differentiable.
}

\rem{
If we restricted our attention to the continuous--time dynamics in~\eqref{eq:dyn}, then Assump.~\ref{asmp:diff} would suffice to provide the local existence and uniqueness of trajectories imposed by Assump.~\ref{asmp:flow};
as illustrated by~\cite[Ex.~2]{Ballard2000},
Assump.~\ref{asmp:diff} is insufficient when the vector field~\eqref{eq:dyn:cont} is coupled to the reset map~\eqref{eq:dyn:disc}.
}

\subsection{Decoupled limbs}
\label{sec:cont}
Since continuity is necessary for differentiability, we must impose a condition that yields continuous outcomes for trajectories of~\eqref{eq:dyn}.
A general condition that is known%
\footnote{We refer to~\cite[Thm.~20]{Ballard2000} for a detailed exposition of this result.}
to provide continuity is that constraint surfaces intersect orthogonally relative to the mass matrix.
Formally,
\eqnn{\label{eq:orth}
\forall i,j\in\set{1,\dots,n},\ i\ne j,\ q\in a_i^{-1}(0)\cap a_j^{-1}(0) :\quad\\
Da_i(q) M(q)^{-1} Da_j(q)^\tr = 0.
}
Physically, this condition implies that any limb or body segments that can undergo impact simultaneously must be inertially decoupled.
Although this condition ensures trajectory outcomes are continuous~\cite[Thm.~20]{Ballard2000}, they generally remain nonsmooth~\cite[Thm.~1]{PaceBurden2017hscc}. 
Thus we introduce a stronger condition that entails decoupling limb forces through a body.

\asmpna{limbs decoupled through body}{
\label{asmp:decoupled}
The configuration decouples into $(n+1)$ segments, hence $2^n$ possible contact modes,  $q = \paren{q_j}_{j=0}^n \in Q = \prod_{j=0}^n Q_j$ where $Q_j = \R^{d_j}$
so that:
\begin{enumerate}
  \item the mass matrix is block diagonal,
$M(q) = \diag\paren{{M_j(q_j)}}_{j=0}^n$, 
where $M_j: Q_j \into \R^{(d_j\times d_j)}$;
\item for limb $j\in\set{1,\dots,n}$
the constraint $a_j$ only depends on $q_j$, $a_j:Q_j\into\R$, 
the coefficient of restitution $\gamma_j$ only depends on the limb states,
$\gamma_j:TQ_j\into\R$,
and
the effort $f_j$ only depends on the states of the limb and the body,
$f_j:TQ_0\times TQ_j\into\R^{d_j}$;
\item the effort $f_0$ applied to the body depends additively on the states of the limbs and the body,
$f_0 = g_0~+~\sum_{j=1}^n g_j$,
where 
for $j>0$, $g_j:TQ_0\times TQ_j\into\R^{d_0}$, and $g_0: TQ_0 \into \R^{d_0}$.
\end{enumerate}
}

\rem{
In the decoupled structure described in the preceding assumption,
the variable $q_0\in Q_0 = \R^{d_0}$ contains the ``body'' degrees--of--freedom, i.e. all coordinates that cannot undergo impact (and are not inertially coupled to those that can).
A limb may contain several links and as such have several bilateral constraints corresponding to it. For instance in \cite[Fig.~1(middle)]{KenneallyDe2016},
one limb contains four rigid bars.
Each limb can be coupled through forces with the body, but can only influence other limbs indirectly through the body.
Note that series compliance~\cite{SprowitzTuleu2013, OdhnerJentoft2014} and/or 
backdrivability~\cite{HyunSeok2014, KenneallyDe2016} 
contribute to inertial decoupling,
but conditions (1) and (2) of {\asmpdecoupled} require inertial decoupling in all degrees--of--freedom between 
body and limbs.
}

\remna{discontinuous outcomes in locomotion}{
The analysis of a saggital--plane quadruped in~\cite{RemyBuffinton2010} provides an instructive example of the behavioral consequences of coupling limbs in legged locomotion.
As summarized in~\cite[Sec~3.1]{RemyBuffinton2010}, the model possesses 3 distinct but nearby trot gaits, corresponding to whether two legs impact simultaneously 
or at distinct time instants%
;
the simultaneous--impact trot is unstable due to discontinuous dependence of trajectory outcomes on initial conditions.
}


\subsection{Differentiable constraint activation/deactivation times}

Trajectories of~\eqref{eq:dyn} are not continuous functions of time due to intermittent impacts that trigger the reset map~\eqref{eq:dyn:disc}.
However, it has been known for some time%
\footnote{%
The earliest instance of this result we found in the English literature is~\cite{AizermanGantmacher1958}.
Subsequently, many authors (ourselves included) have re--proven this result; a partial list includes~\cite{HiskensPai2000, GrizzleAbba2002, WendelAmes2012, BurdenRevzen2015tac}.
The result follows via a straightforward composition of smooth flows with smooth time--to--impact maps; we refer the interested reader to~\cite[App.~A1]{BurdenRevzen2015tac} for details.
}
that trajectory outcomes can nevertheless depend differentiably on initial conditions away from impact times, so long as the contact mode sequence is fixed.
For this result to hold, the time when constraints activate (or deactivate) must depend differentiability on initial conditions.
We now develop definitions used to state an \emph{admissibility} condition at the end of the section that yields differentiable time--to--activation (and time--to--deactivation).

\defna{admissible constraint activation/deactivation}{
\label{def:adact}
A trajectory initialized at $(q,\dot{q})\in TA_J\subset TQ$
\emph{activates constraints $I\in 2^n$ at time $t > 0$} if
(i) no constraint in $I$ was active immediately before time $t$
and
(ii) all constraints in $I$ become active at time $t$;
this activation is \emph{admissible} if the constraint velocity%
\footnote{Formally, the \emph{Lie derivative}~\cite[Prop.~12.32]{Lee2012} of the constraint along the vector field specified by~\eqref{eq:dyn:cont}.
}
for all activated constraints is negative.
Formally, with $(\rho,\dot{\rho}^-) = \lim_{s\goesto t^-}\phi(s,(q,\dot{q}))$ denoting the left--handed limit of the trajectory at time~$t$,
\eqnn{
\forall i\in I: D_t\brak{a_i\circ\phi}(0,(\rho,\dot{\rho}^-)) = Da_i(\rho)\dot{\rho}^- < 0.
}
Similarly, the trajectory 
\emph{deactivates constraints $I\in 2^n$ at time $t > 0$} if
(i) all constraints in $I$ were active at time $t$
and
(ii) no constraint in $I$ remains active immediately after time $t$;
this deactivation time is \emph{admissible} if, for all deactivated constraints: 
(i) the constraint velocity or constraint acceleration%
\footnote{Formally, the second Lie derivative of the constraint along the vector field specified by~\eqref{eq:dyn:cont}.}
is positive,
or
(ii) the time derivative of the contact force is negative.
Formally, with $(\rho,\dot{\rho}^+) = \lim_{s\goesto t^+}\phi(s,(q,\dot{q}))$ denoting the right--handed limit of the trajectory at time $t$,
for all $i\in I:$
\eqnn{\label{eq:deact}
\text{(i)}\ D_t\brak{a_i\circ\phi}(0,(\rho,\dot{\rho}^+)) &> 0\ \text{or}\\
D^2_t\brak{a_i\circ\phi}(0,(\rho,\dot{\rho}^+)) &> 0,\\ 
\text{or (ii)}\ D_t\brak{\lambda_i\circ\phi}(0,(\rho,\dot{\rho}^+)) &< 0. 
}
}

\rem{\label{rem:adact}
The conditions for admissible constraint deactivation in case (i) of~\eqref{eq:deact} can only arise at constraint activation times; otherwise the trajectory is continuous, whence active constraint velocities and accelerations are zero.
}


\defna{admissible trajectory}{
\label{def:trj}
The trajectory initialized at $(q,\dot{q})$ is \emph{admissible on $[0,t]\subset\R$} if 
(i) it has a finite number of constraint activation (hence, deactivation) times on $[0,t]$, 
and 
(ii) every constraint activation and deactivation is admissible;
otherwise the trajectory is \emph{inadmissible}.
}

\defna{contact mode sequence}{
\label{def:seq}
The \emph{contact mode sequence}%
\footnote{This definition differs from the \emph{word} of~\cite[Def.~4]{JohnsonBurden2016ijrr} in that a contact mode is included in the sequence only if nonzero time is spent in the mode; this definition is more closely related to the \emph{words} of~\cite[Eqn.~72]{BurdenSastry2016siads}} 
associated with an admissible trajectory $\phi^{(q,\dot{q})}$ on $[0,t]\subset\R$ that has $m$ activation and/or deactivation times $t_1, \dots, t_m$
is the unique function $\omega:\set{0,\dots,m}\into 2^n$ such that there exists a finite sequence of times $\set{t_\ell}_{\ell=0}^{m+1}\subset [0,t]$ for which $0 = t_0 < t_1 < \cdots < t_{m+1} = t$ and
\eqnn{
\forall \ell\in\set{0,\dots,m} : \phi((t_{\ell},t_{\ell+1}),(q,\dot{q})) \subset TA_{\omega(\ell)}.
}
}

\rem{
\label{rem:seq}
In~{\defseq}, the sequence $\omega$ is easily seen to be unique by the admissibility of the trajectory; indeed, the associated time sequence consists of start, stop, and constraint activation/deactivation times.
}

\asmpna{admissible trajectories}{
\label{asmp:trj}
The trajectory of~\eqref{eq:dyn} initialized at $(q,\dot{q})$ is admissible on $[0,t]$ for all $(t,(q,\dot{q}))\in\e{F}$.
}

\section{Differentiability through contact}
\label{sec:diff}

Under Assumptions~1--4 from~\sct{asmp}, previous work has shown that, when the contact mode sequence is fixed, trajectory outcomes vary continuously~\cite[Thm.~20]{Ballard2000} and differentiably~\cite{AizermanGantmacher1958} with respect to variations in initial conditions (i.e. initial states and parameters).
This enables the use of scalable algorithms for optimal control~\cite{Polak1997} and reinforcement learning~\cite{SuttonBarto1998} to improve the performance of a given behavior (corresponding to the fixed contact mode sequence) using gradient descent.
However, these algorithms cannot be relied upon to select among different behaviors (corresponding to different contact mode sequences) since trajectory outcomes are known to depend nonsmoothly on initial conditions~\cite[Thm.~1]{PaceBurden2017hscc}.
In this section we report that decoupled limbs yield classically differentiable trajectory outcomes \emph{even as the contact mode sequence varies},
enabling the use of scalable algorithms to select behaviors.

\begin{restatable}[differentiability through intermittent contact]{theorem}{diffcontact}
\label{thm:diff}
Under Assumptions 1--4 from~\sct{asmp},
if 
$t$ is not a constraint activation time for $(\q,\dot{\q})$,
then
the flow $\phi:\e{F}\into TA$ for~\eqref{eq:dyn}
is continuously differentiable 
at 
$(t,(\q,\dot{\q}))\in\e{F}$.
\end{restatable}

\remna{proof sketch}{

\iftoggle{arxiv}
{
We provide an illustration of the result in~\fig{hds}, and a sketch of the proof strategy in what follows.
For the complete proof, see~Sec.~\ref{sec:proof}.
}
{
Due to space constraints, we relegate the formal proof of this result to a technical report~\cite[Thm.~1]{PaceBurden2016arxiv}.
In its stead, we provide an illustration of the result in~\fig{hds}, and a sketch of the proof strategy in what follows.
}
Given a contact mode sequence $\omega$ for a trajectory initialized near $(\q,\dot{\q})$, we construct a continuously differentiable ($C^1$) function $\phi_\omega$ defined on an open set containing $(t,(q,\dot{q}))$ by composing the sequence of flow--to--activation and flow--to--deactivation functions specified by $\omega$.
Without loss of generality, we only consider constraint activations.%
\footnote{Admissible constraint deactivations do not alter the flow to first order since the state and vector field are continuous during these transitions.}
Near $(\q,\dot{\q})$ in~\fig{hds}, there are two activation sequences, corresponding to whether constraint $1$ activates before constraint $2$ activates, or vice--versa.
For each $I\subset\set{1,2}$ we let $\phi_I$ denote the $C^1$ flow for~\eqref{eq:dyn:cont},%
\footnote{These flows are guaranteed to exist over an open subset of $\R\times TQ$ containing $\set{0}\times A_I$ by~{\asmpdiff}.}
and define the $C^1$ function $\Gamma_I(u,(\p,\dot{\p})) = (u,(\p,\Delta_I(\p)\dot{\p}))$.
By~{\asmptrj}, 
there exist $C^1$ time--to--activation functions $\tau_{\set{1}}^2, \tau_{\emptyset}^2$ for constraint $2$ defined over open neighborhoods of $(\rho,\dot{\rho}^-)$ and 
$\paren{\rho, \dot{\rho}^+_{\set{1}}}$ and similarly 
there exists $C^1$ time--to--activation functions $\tau_{\set{2}}^1,\tau_{\set{\emptyset}}^1$ for constraint $1$ defined over open neighborhoods of
$(\rho,\dot{\rho}^-)$ and 
$\paren{\rho, \dot{\rho}^+_{\set{2}}}$.
For each contact mode $I\subset\set{1,2}$ and constraint $j\in\set{1,2}$ undergoing activation ($j\not\in I$), we let $\vphi_I^j$ denote the flow--to--activation,
\eqnn{
\vphi_{I}^j(u,(\p,\dot{\p})) = \paren{u - \tau_{I}^j(\p,\dot{\p}), \phi_{I}(u - \tau_{I}^j(\p,\dot{\p}), (\p,\dot{\p}))};
}
since $\vphi_{I}^j$ is obtained via composition from $C^1$ functions, it is a $C^1$ function.
For 
$\omega_1 = \paren{\emptyset,\set{1},\set{1,2}}$, the function $\phi_{\omega_1}$ is given by the composition
\eqnn{
\phi_{\omega_1} = \phi_{\set{1,2}}\circ \Gamma_{\set{2}}\circ\vphi_{\set{1}}^2\circ\Gamma_{\set{1}}\circ\vphi_{\emptyset}^1;
}
for 
$\omega_2 = \paren{\emptyset,\set{2}, \set{1,2}}$, the function $\phi_{\omega_2}$ is given by the composition
\eqnn{
\phi_{\omega_2} = \phi_{\set{1,2}}\circ \Gamma_{\set{1}}\circ\vphi_{\set{2}}^1\circ\Gamma_{\set{2}}\circ\vphi_{\emptyset}^2.
}
Since both $\phi_{\omega_1}$ and $\phi_{\omega_2}$ are obtained via composition from $C^1$ functions, they are $C^1$ functions.
The generalization of this procedure to arbitrary contact mode sequences is given in~
\iftoggle{arxiv}
{
Sec.~\ref{sec:proof}.
}
{
\cite[Proof of Thm.~1]{PaceBurden2016arxiv}.
}
As illustrated in~\fig{hds}, the trajectory outcome near $\phi(t,(\q,\dot{\q}))\in TA_{\set{1,2}}$ is differentiable with respect to the initial condition near $(\q,\dot{q})\in TA_{\emptyset}$, even as the contact mode sequence changes from $\omega_1$ to $\omega_2$.
Formally, we can show that
$D\phi_{\omega_1}(t,(\q,\dot{\q})) = D\phi_{\omega_2}(t,(\q,\dot{\q}))$
by computing these derivatives via the Chain Rule; this entails taking products of matrices with the general form%
\footnote{For the definition of $h_j$, see~\eqref{eq:hi}; for the definition of $F_I$, see~\eqref{eq:vecfield}.}
\eqnn{
D\Gamma_{I}(u,(\p,\dot{\p}))  = \mat{ccc}{
	1 & 0 & 0 \\
	0 & \id & 0 \\
	0 & D_\p(\Delta_{I}(\p,\dot{p})\dot{\p}) & D_{\dot{p}}\paren{\Delta_{I}(\p,\dot{p})\dot{p}} },
}
\eqnn{
D\vphi_{I}^{j}(u,(\p,\dot{\p})) = \mat{cc}{
	1 & \frac{1}{Dh_{j}(\p,\dot{p})F_{I}(\p,\dot{p})}Dh_{j}(\p,\dot{p}) \\
	0 & \II-\frac{1}{Dh_{j}(\p,\dot{p})F_{I}(\p,\dot{\p})}F_{I}(\p,\dot{\p})Dh_{j}(\p,\dot{p})
}.
}
}

\begin{figure}[th]
	\vspace{.5em}
\centering
{
\iftoggle{ieee}
{
\def\svgwidth{1.5\columnwidth} 
\resizebox{\columnwidth}{!}{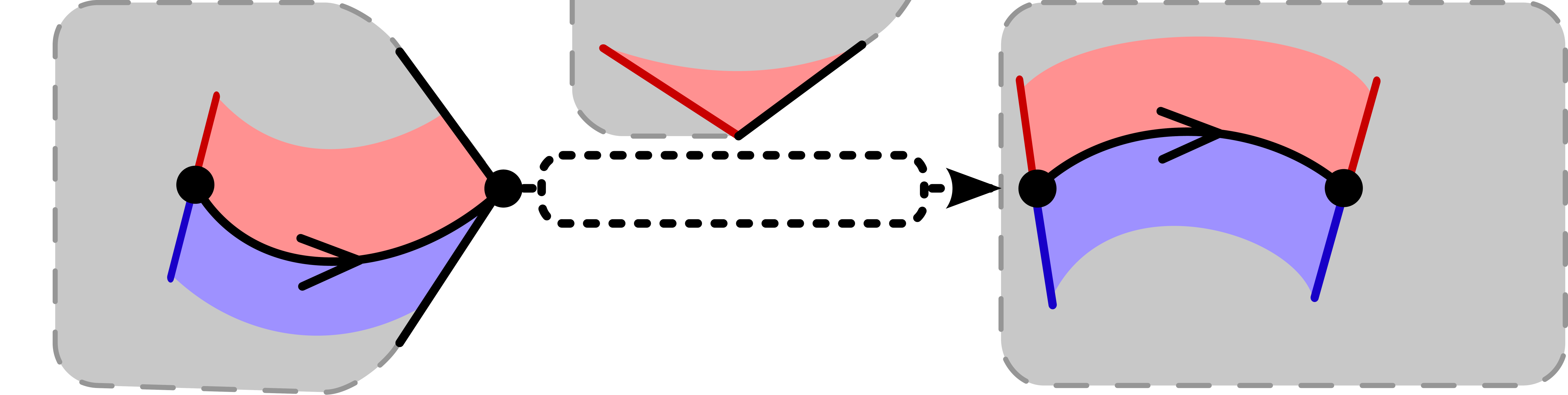}
}
{
\def\svgwidth{0.95\columnwidth} 
\resizebox{\columnwidth}{!}{\input{hds.pdf_tex}}
}
}
\caption{\label{fig:hds}
Illustration of trajectory undergoing two simultaneous
constraint activations:
the trajectory initialized at $(\q,\dot{\q})\in TA_{\emptyset}\subset TQ$
flows via~\eqref{eq:dyn:cont} to a point $(\rho,\dot{\rho}^-)\in TA_{\emptyset}$
where both 
constraint functions $a_1$, $a_2$ are zero,
instantaneously resets velocity via~\eqref{eq:dyn:disc} to $\dot{\rho}^+ = \Delta_{\set{1,2}}(\rho,\dot{\rho}^-)$,
then flows via~\eqref{eq:dyn:cont} to $\phi(t,(\q,\dot{\q}))\in TA_{\set{1,2}}\subset TQ$.
Nearby trajectories undergo activation and deactivation at distinct times: 
trajectories initialized in the red region activate constraint $1$ and flow through contact mode $TA_{\set{1}}$ before activating constraint $2$---their contact mode sequence is $\omega_1 = \paren{\emptyset,\set{1},\set{1,2}}$---while
trajectories initialized in the blue region activate $2$ and flow through $TA_{\set{1,2}}$ before deactivating $1$---their contact mode sequence is $\omega_2 = \paren{\emptyset,\set{2},\set{1,2}}$.
Differentiability of trajectory outcomes is illustrated by the fact that red outcomes lie along the same submanifold as blue.
}
\end{figure}


\section{Discussion}
\label{sec:disc}


We conclude by discussing implications and routes to generalizing the theoretical results reported above.

\subsection{Implications for optimization and learning}
\label{sec:opt}
Optimization and learning algorithms have emerged in recent years as powerful tools for synthesis of dynamic and dexterous robot behaviors~\cite{MombaurLongman2005,Todorov2011,KuindersmaDeits2015,LevineFinn2016,KumarTodorov2016}.
Since scalable algorithms leverage derivatives of trajectory outcomes, their applicability to the dynamics in~\eqref{eq:dyn} has previously
(i) been confined to a fixed contact mode sequence~\cite{MombaurLongman2005, MombaurBockTRO2005} 
or 
(ii) relied on approximations or relaxations of the dynamics~\cite{Todorov2011,KuindersmaDeits2015,LevineFinn2016,KumarTodorov2016}.
Neither of these approaches is entirely satisfying:
(i) prevents the algorithm from automatically selecting the behavior (corresponding to the contact mode sequence) in addition to extremizing its performance;
(ii) implies the model under consideration is no longer a mechanical system subject to unilateral constraints.
The results we report in~\sct{diff} provide an analytical and computational framework within which derivative--based algorithms can be rigorously and directly applied to the dynamics of mechanical systems subject to unilateral constraints~\eqref{eq:dyn} to select between permutations of constraint (de)activations.

\subsection{Decoupled limbs}
\label{sec:decoupled}
{\asmpdecoupled} can be interpreted physically as asserting that robot segments that can undergo impact simultaneously (i.e. limbs) must be decoupled through another segment not undergoing impact (i.e. the body).
Crucially, this condition is required to ensure trajectory outcomes vary continuously with respect to initial conditions~\cite[Thm.~20]{Ballard2000};
since continuity is a precondition for differentiability, this condition is equally necessary for the result reported in~{\thmdiff}.
We note that this condition is violated by conventional robots constructed from rigid serial chains and non--backdrivable actuators~\cite{MurrayLi1994}.
In contrast, design methodologies that incorporate direct--drive actuators~\cite{HyunSeok2014, KenneallyDe2016} or series compliance~\cite{SprowitzTuleu2013, OdhnerJentoft2014} tend to produce robot locomotors and manipulators with limbs that are (approximately) decoupled.
How approximately the limbs are decoupled is the determining factor on whether {\asmpdecoupled} holds, and hence
whether the trajectories are differentiable with respect to initial conditions away from (de)activations.

\subsection{Grazing contact}
\label{sec:grazing}


{\defadact} precludes \emph{grazing} trajectories, 
i.e. those that activate constraints with zero constraint velocity, or deactivate constraints with zero instantaneous rate of change in contact force.
The key technical challenge entailed by allowing constraint activation (resp. deactivation) we termed \emph{inadmissible} lies in the fact that the time--to--activation (resp. time--to--deactivation) function is not differentiable.
This fact 
has been shown by others~\cite[Ex.~2.7]{Di-BernardoBudd2008}, and 
is straightforward to see in an example.  
Indeed, consider the trajectory of a point mass moving vertically in a uniform gravitational field subject to a maximum height (i.e. ceiling) constraint.
The grazing trajectory is a parabola, whence the time--to--activation function involves a square root of the initial position.

\subsection{Zeno phenomena}
\label{sec:zeno}


{\defadact} precludes \emph{Zeno} trajectories, 
i.e. those that undergo an infinite number of constraint activations (hence, deactivations) in a finite time interval.
The key technical challenge entailed by allowing
Zeno
lies in the fact that
evaluating the flow requires composing an infinite number of flow--and--reset functions.
Composing a finite number of smooth functions yields a smooth function, but the same is not generally true for infinite compositions.
Thus although
it is possible to show that the infinite composition results in a differentiable flow in simple examples like the \emph{rocking block}~\cite{Housner1963} and \emph{bouncing ball}~\cite[Sec.~6.1]{Ballard2000}, 
we cannot at present draw any general conclusions regarding differentiability of the flow along Zeno trajectories.

\subsection{Friction}
\label{sec:friction}


Friction is a microscopic phenomenon that eludes first--principles understanding~\cite{GerdeMarder2001}.
Phenomenological models of friction are macroscopic approximations;
one popular model%
\footnote{Usually attributed to Coulomb, but also due to Antomons~\cite{GerdeMarder2001}.}
posits a transition from \emph{sticking} to \emph{sliding} when the ratio of normal to tangential force drops below a parameterized threshold.
The system's flow is discontinuous at this threshold, as some trajectories \emph{slide} away from their \emph{stuck} neighbors.
Even if such transitions are avoided, the introduction of simple friction models into mechanical systems subject to unilateral constraints is known to produce pathologies including nonexistence and nonuniqueness of trajectories~\cite{Stewart2000}.

\subsection{Non--Euclidean configuration spaces}
\label{sec:noneuc}


We restricted the configuration space to $Q = \R^d$ starting in~\sct{mdl} to simplify the exposition and lessen the notational overhead.
Nevertheless, the preceding results apply to configuration spaces that are complete Riemannian manifolds.%
\footnote{Since the preceding results are not stated in coordinate--invariant terms, they are formally applicable only after passing to coordinates.}

\subsection{Contact--dependent effort}
\label{sec:nonauto}


The dynamics in~\eqref{eq:dyn} vary with the contact mode $J\subset\set{1,\dots,n}$ due to intermittent activation of unilateral constraints $a_J(q) \ge 0$,
but the (so--called~\cite{Ballard2000}) effort map $f$ was not allowed to vary with the contact mode.
Contact--dependent effort can easily introduce nonexistence or nonuniqueness.
Indeed, consider a 
planar system with $q\in\R^2$ 
undergoing plastic impact with the constraint surface specified by $a(q) = q_1$ 
subject to contact--dependent effort that satisfies $f_\emptyset(q) = (-1,+1)$ if $q_1 > 0$ and $f_{\set{1}}(q) = (+1,-1)$ if $q_1 = 0$.
Every trajectory eventually activates the constraint.
Once the constraint is active, the trajectory becomes ill--defined.

\subsection{Massless limbs}
\label{sec:massless}


To accommodate massless limbs, one must specify their unconstrained dynamics.
If the unconstrained dynamics differ from the constrained dynamics, then in effect one has introduced contact--dependent effort, whence we refer to the preceding section.
If the unconstrained dynamics do not differ from the constrained dynamics, then in effect one has introduced bilateral constraints the massless limbs must satisfy, whence we refer to the subsequent section.
The constrained dynamics of massless limbs are derived in~\cite{Brogliato2015}.

\subsection{Bilateral constraints}
\label{sec:constraints}


The preceding results hold in the presence of bilateral (i.e. equality) constraints 
so long as they do not couple limbs.
Formally, if the bilateral constraints $b(q) = 0$ are specified by a differentiable function $b:Q\into\R^m$,
there must exist an assignment $\beta:\set{1,\dots,m}\into\set{1,\dots,n}$
such that 
for all 
bilateral constraints $k\in\set{1,\dots,m}$,
unilateral constraints $i,j\in\set{1,\dots,n}$, $i\ne j$, 
and configurations $q\in b^{-1}(0)\cap a_i^{-1}(0)\cap a_j^{-1}(0)$:
\eqnn{\label{eq:orthb}
\quad\ip{Da_i(q)}{Da_j(q)}_{M^{-1}} &= 0, \\
\quad\ip{Db_{\beta(i)}(q)}{Da_j(q)}_{M^{-1}} &= 0.
}

\subsection{Non--autonomous dynamics}
\label{sec:nonauto}


One may wish to allow the continuous and/or discrete dynamics in~\eqref{eq:dyn} to vary with time or an external input. 
Some common cases can easily be handled.
If the dynamics are time--varying, but time could be incorporated as a state variable so that the preceding assumptions hold for the augmented system determined by $\td{q} = (t,q)\in \td{Q} = \R\times Q$,
\eqnn{
\td{M}\paren{\td{q}} = \diag\paren{1, M(q)},\ \td{f}\paren{\td{q},\dot{\td{q}}} = (0,f(t,q,\dot{q})),
}
then the preceding results apply directly to the augmented system;
a similar observation holds when the value of an external input is determined by time and state in such a way that the closed--loop system (possibly augmented as above to remove the time dependence) satisfied the preceding assumptions.

\iftoggle{ieee}{
}
{

}

\iftoggle{arxiv}{
\section{Appendix: Proof of \thmdiff }
\label{sec:proof}

\diffcontact*

\begin{proof}
We begin with an apology to the reader. 
The notation used in this proof is nonstandard; it is our hope that though it is nonstandard
the notation clairifies the steps more than it confuses the reader.

Before begining with the proof we introduce some notation:
\begin{enumerate}
\item $x[i]$ denotes the $i$th entry into variable $x$;
	
	
\item $q$ is the vector containing all limb positions. $q = \left[q[i]\right]_{i=0}^n \in Q = \Pi_{i=0}^n Q_i$; 
\item similarly $\dot{q}$ is the vector containing all limb velocities. $\dot{q} = \left[\dot{q}[i]\right]_{i=0}^n$. $(q,\dq)\in TQ$;
\item $\bar{a}_i\colon Q\into~\R$ and $h_i\colon TQ \into~\R$
	where 
	\begin{align}
		\forall \q \in Q &\qquad \bar{a}_i(q)=a_i(q_i),\\
		\forall (\q,\dot{\q}) \in TQ &\qquad h_i(q,\dot{q}) = a_i(q_i), \label{eq:hi}
	\end{align}
	the logical extensions of $a_i$ to the corresponding domains;
\item let $\Box_J\colon TQ \into \R^{d\times d}$ where $\Box_J(q,\dot{q}) = D_q(\Delta_J(q, \dot{q})\dot{q})$, 
	the derivative of the post-impact velocity with respect to position;
\item let $\Diamond_J\colon TQ \into \R^{d\times d}$ where $\Diamond_J(q,\dot{q}) = D_{\dot{q}}(\Delta_J(q, \dot{q})\dot{q})$, 
	the derivative of the post-impact velocity with respect to position;%
	\footnote{Given the coefficient of restitution $\gamma$ is dependent upon $\dot{q}$, $\Diamond_J(q,\dot{q})$ might differ from $\Delta_J(q,\dot{q})$.}
\item $\mat{c}{q[i] \\ \dot{q}[i]}_{i=0}^n=\mat{c}{q[0]^T, q[1]^T, \cdots, q[n]^T, \dot{q}[0]^T, \cdots, \dot{q}[n]^T}^T$ is the deinterleaving of the position and velocity components of the individual limbs into a vector where the first half corresponds to position components and the later half corresponds to velocity values.
\end{enumerate}
What follows are some helpful identies based upon the above notation and the given assumptions.
\begin{enumerate}
\item Given orthogonality of constraints:%
\footnote{ Properties (a)-(c) are included here for completeness and (d)-(e) may also be seen more succinctly 
		using the limb decoupling assumption. }
\begin{enumerate}
	\item the reset map has a block diagonal form with 
		$\Delta^j_J\colon TQ_j \into \R^{d_j\times d_j}$, the reset map for limb $j$ with $J$ constraints active:
		\eqnn{
			\Delta_J(q[j], \dot{q}[j]) = \diag\paren{\Delta_J^{{j}}(q[j],\dot{q}[j])}_{j=0}^n, }
		where
		\eqnn{
	\Delta&^{{j}}_J(q[j],\dot{q}[j]) \\
	&= I_{d_j} - \paren{1+\gamma_j(q[j], \dot{q}[j])}M_j(q[j])^{-1}Da_j^T(q[j])\paren{Da_j(q[j]) M(q[j])^{-1} Da_j^T(q[j])}Da_j(q[j]);
		}

	\item the velocity of a given limb is not affected by a reset if the constraint corresponding to that limb is not active:
		\eqnn{\Delta_J^{k}(q[k], \dot{q}[k]) = I_{d_k} \text{ if }k \notin J;}
	\item the reset map for constraint mode $J$ is equal to the product of the reset maps for each active constraint:
		\eqnn{\label{eq:deltaprod}\prod_{j\in J } \Delta_{\set{j}}(q, \dot{q}) = \Delta_J(q, \dot{q});}

	\item $\Box_J(q,\dot{q})$ has a block diagonal structure with 
		$\Box_J^j\colon TQ_j\into \R^{d_j\times d_j}$:
		\eqnn{ \Box_J(q,\dot{q}) = \diag\paren{\Box_J^{{j}}(q[j], \dot{q}[j])}_{j=0}^n,}
		where
		\eqnn{
		\Box_J^j(q[j],\dot{q}[j]) = D_{q_j}\paren{\Delta_J^j(q[j],\dot{q}[j]}\dot{q}[j]);
	}
\item if a given limb is not in the active constraint set, the block corresponding to the limb in $\Box_J$ is zero:
	\eqnn{\Box_J^{{k}}(q[k],\dot{q}[k]) = 0 \text{ if }k \notin J.}
\item $\Box_J(q,\dot{q})$ is the summation of $\Box_{\set{j}}$ for each active constraint:
	\eqnn{\label{eq:boxsum}\sum_{j\in J} \Box_{\set{j}}(q, \dot{q}) = \Box_J(q,\dot{q}).}
\end{enumerate}
\item Given limb decoupling: 
	\begin{enumerate}
	\item the acceleration of each limb is only dependent upon the given limb's and the body's current position and velocity; $\alpha_J \colon TQ_J \into \R^d$:
\eqnn{
	\label{eq:accel}
	\alpha_J(q,\dot{q})[j]
	= \begin{cases}
		 M_j^{-1}(q[j])\paren{f_j(q[0],\dot{q}[0], q[j], \dot{q}[j])+c_j(q[j],\dot{q}[j])\dot{q}} & \text{ if $j\notin J$ and $j\neq 0$} \\ 
		\begin{aligned} 
		 M_j^{-1}(q[j])\left(f_j(q[0],\dot{q}[0], q[j], \dot{q}[j])+c_j(q[j],\dot{q}[j])\dot{q} + \right .\\
		 \left. Da_j(q[j])^T\lambda_j(q[j], \dot{q}[j]) \right) 
	  \end{aligned} & \text{ if $j \in J$ } \\
		 M_0^{-1}(q[0])\paren{f_0(q,\dot{q})+c_0(q[0], \dot{q}[0])\dot{q}[0]} & \text{ if $j=0$;}
	\end{cases}
}

\item $\Diamond_J(q,\dot{q})$ has a block diagonal form with $\Diamond_J^i\colon TQ_i \rightarrow \R^{d_j\times d_j}$:
\begin{align}
	\Diamond_J^j(q[j],\dot{q}[j]) &= D_{\dot{q}[j]}\Delta_J^j(q[j],\dot{q}[j])\dot{q}[j] \\
	D_{\dq}(\Delta_J(q,\dot{q})\dot{q}) &= D_{\dot{q}}\left(\diag[\Delta_J^j(q[j],\dq[j])]_{j=0}^n\right), \\
		&= \diag\left[D_{\dq[j]}\Delta_J^j(q[j],\dq[j])\dq[j]\right]_{j=0}^n \\
		&= \diag[ \Diamond_J^j(q[j],\dq[j])]_{j=0}^n \label{eq:diagDiamond};
\end{align}
\item if a given limb is not in the active constraint set, the corresponding block in $\Diamond_J$ is the identity:
\eqnn{
	\Diamond_J^k(q[k],\dq[k]) = I_{d_k} \text{ if }k \notin J;
}
\item $\Diamond_J(q,\dot{q})$ is the product of $\Diamond_{\set{j}}$ for each active constraint:
	\eqnn{\label{eq:boxsum}\prod_{j\in J} \Diamond_{\set{j}}(q, \dot{q}) = \Diamond_J(q,\dot{q}).}

\end{enumerate}
\end{enumerate}

We now proceed with the proof.
\begin{enumerate}

	\item We repeat some of the notations found in the proof of~\cite[Thm.~1]{PaceBurden2017hscc} here.
		\begin{enumerate}

\item For any given perturbation, there is a finite set of selection functions corresponding to a sequence of (de)activating constraints. 
\item These selection functions will be indexed by a pair of functions $(\omega,\eta)$ where: 
$\omega\colon \set{0,\dots,m}\into 2^n$ is a contact mode sequence, 
i.e. $\omega\in\Omega$;
$\eta\colon \set{0,\dots,m-1}\into\set{1,\dots,n}$ indexes constraints that undergo admissible activation or deactivation.
\item Let $\mu \colon \set{0,\dots,m}\into 2^n$ be defined as 
$\mu(k) = \bigcup_{i=0}^{k-1}\set{\eta(i)}$, where we adopt the convention that $\bigcup_{i=0}^{-1}\set{i} = \emptyset$;
note that $\mu$ is uniquely determined by $\eta$.
As in~\cite[Thm.~1]{PaceBurden2017hscc}, we suppress notation indicating dependence on $\w$ and $\eta$ until \eqref{eq:salt_final}.
\item Let $(\rho,\dot{\rho}^-)=\lim_{u\uparrow s}\phi(u,(\q,\dot{\q}))$.
For all $k\in\set{0,\dots,m}$, define $\dot{\rho}_{k} = \Delta_{\mu(k)}(\rho)\dot{\rho}^-$,
 once again using the convention $\bigcup_{i=0}^{-1} = \emptyset$.
\end{enumerate}
\item Since assumptions 1--4 from~\sct{asmp}\ satisfy the hypotheses of~\cite[Thm.~1]{PaceBurden2017hscc},%
	\footnote{\asmpdecoupled is stronger than the orthoganality of constraints assumption.}
	$\phi$ is piecewise--differentiable.

\item Let $K_a$ be the set of constraints undergoing activation and $K_d$ be the set of constraints undergoing deactivation.\footnote{Activating constraints may instantaneously deactivate, e.g. a bean bag hitting a ceiling.}
\item Assume, without loss of generality, the trajectory begins in the unconstrainted mode $\emptyset$.%
\footnote{This does not imply $K_d=\emptyset$.}

\item Let $R_J$ denote the reset map into constraint mode $J$;
	$R_J\colon TQ \into TQ_J$, 
	\eqnn{\forall(q,\dot{q})\in TQ\colon\qquad R_J(q,\dot{q})= \mat{c}{q \\ \Delta_J(q, \dot{q})\dot{q}}.}
\item Let $DR_J$ denote the (Jacobian) derivative of $R_J$ with respect to $(q,\dot{q})$.
	\eqnn{
		\forall (q,\dot{q}) \in TQ: \qquad
		DR_J(q,\dot{q}) = \mat{cc}{ \I & 0 \\ \Box_J(q,\dot{q}) & \Diamond_J(q, \dot{q}) }.}
\item Let $\widetilde{R}_J^L$ denote the reset map from constraint mode $J$ into constraint mode 
	$J\backslash L$; $\widetilde{R}_J^L\colon TQ_J~\rightarrow~TQ_{J\backslash L}$,
	\eqnn{
		\forall(q,\dot{q}) \in TQ_J: \qquad
		\widetilde{R}_J^L(q,\dot{q}) = \begin{bmatrix} q \\ \dot{q} \end{bmatrix}.
	}
\item Let $\hat{R}_\ell  \colon TQ \into TQ,$ 
	\eqnn{
		\forall(q,\dot{q}) \in TQ:
		\qquad \hat{R}_\ell(q,\dot{q}) = \begin{cases} R_{\set{\eta(\ell)}}(q,\dot{q}) & \text{ if } \eta(\ell) \in K_a 
		\vspace{1mm}\\
		\widetilde{R}_{\omega(\ell)}^{\set{\eta(\ell)}}(q,\dot{q}) & \text{ if } \eta(\ell) \in K_d.
	\end{cases}
	}
\item Let $F_J$ denote the vector field in constraint mode $J$. $F_J\colon TQ_J \into \R^{2d}$, 
	\eqnn{
		\label{eq:vecfield}
		\forall(q,\dot{q}) \in TQ_J:\qquad F_J(q,\dot{q})=\mat{c}{\dot{q}\\ \alpha_J(q,\dot{q})},}
where $\alpha_J(q,\dot{q})$ is defined in~\eqref{eq:accel}.
\newcommand{\rhodotm}{\dot{\rho}_{\ell}}
\newcommand{\rhodotp}{\dot{\rho}_{(\ell+1)}}
\item Let 
\eqnn{
	\label{eq:S_l}
	S_\ell =\\ 
		&\frac{1}{Dh_\ell(\rho,\rhodotm) F_{\ell}(\rho,\rhodotm)}
	\paren{F_{\ell+1}(\rho,\rhodotp)-
		D\hat{R}_{\ell}(\rho,\rhodotm)F_{\ell}(\rho,\rhodotm)} 
	Dh_{\ell}(\rho,\rhodotm),
}
where $F_\ell = F_{\w(\ell)}$, 
and $Dh_\ell = Dh_{\eta(\ell)}$; $S_\ell \in \R^{2d\times 2d}$.
\item The saltation matrix for a given word $\w$ at $(\rho, \dot{\rho})$ is 
\eqnn{
	\label{eq:salt_general}
	\Xi_\w^\eta = 
		\prod_{\ell=0}^{\abs{K_a}+\abs{K_d}-1}\left(D\hat{R}_{\w(\ell)\bigcup \set{\eta(\ell)}}(\rho,\rhodotm) + 
			S_\ell\right)
}
where $S_\ell$ is defined in~\eqref{eq:S_l} \cite[Eq.~66]{BurdenSastry2016siads}, \cite[Eq.~2.5]{Ivanov1998}.

\item Given that the vector field associated with a constraint undergoing deactivation is continuous, the corresponding saltation matrix is $\II$. 
That is $S_\ell = 0$ and $D\hat{R}_\ell(\rho,\dot{\rho}_{(\ell-1)}) = \II$ when $\eta(\ell)$ is a deactivation; 
$\eta(\ell) \in K_d$.
In what follows, the calculations are performed only for activating constraints.\footnote{This exclusion does not apply to the case of deactivations caused by an activation, e.g. a bouncing ball or a bean bag hitting the ceiling.}

%

\item The inner product $Dh_{\ell}(\rho,\rhodotm) F_{\ell}(\rho,\rhodotm)$ in the computation of $S_\ell$ is independent of the word~$\w$.
	\eqnn{\label{eq:salt_scaling}
		Dh_{\ell}(\rho,\rhodotm) F_{\ell}(\rho,\rhodotm) 
		= D\bar{a}_{\eta(\ell)}(\rho)\rhodotm 
		= Da_{\eta(\ell)}\paren{\rho[\eta(\ell)]}\dot{\rho}[\eta(\ell)].}

\item The reset map into contact mode $\w(\ell)\bigcup \set{\eta(\ell)}$ from contact mode $\w(\ell)$ 
is the same as the reset map into contact mode $\{\eta(\ell)\}$ from contact mode $\w(\ell)$.%
	\footnote{It is important to note that it is not always the case $\w(\ell+1)=\w(\ell)\bigcup\set{\eta(\ell)}$ as in the case of instantaneous constraint deactivation dependent upon a constraint activation. 
In this case $\w(\ell+1) = \w(\ell)$ and $\eta(\ell) \neq \emptyset$.}
	\begin{align}
		\forall (q, \dot{q}) \in TQ_{\w(\ell)}: \qquad R_{\w(\ell)\bigcup\set{\eta(\ell)}}(q,\dot{q}_\ell) &= \mat{c}{q \\ \Delta_{\w(\ell)\bigcup\set{\eta(\ell)}}(q, \dot{q}_\ell)\dot{q}_\ell} \\
		&= \mat{c}{q \\ \Delta_{\set{\eta(\ell)}}(q, \dot{q}_\ell)\dot{q}_\ell} \\
		&=R_{\{\eta(\ell)\}}(q,\dot{q}_\ell).
	\end{align}

\item From the chain rule for total derivatives \cite[Prop~C.3]{Lee2012} and \asmpdecoupled,
the first order approximation of the reset map into constraint mode $J$ is the same as the product of the reset maps into constraint mode $\set{j}\in J$ at a given point $(q,\dot{q}$).
	\eqnn{ \label{eq:DRI}
		DR_J(q,\dot{q})= D\prod_{j\in J}R_{\{j\}}(q,\dot{q}) = \left(\prod_{j\in J}DR_{\{j\}}\right)(q,\dot{q}).}
\item  Given \eqref{eq:DRI} and identities \eqref{eq:boxsum} and \eqref{eq:diagDiamond},
\eqnn{
	DR_J(q,\dot{q}) = \mat{cc}{\I & 0 \\ \sum_{j\in J}\Box_{\{j\}}(q) & \prod_{j\in J} \Diamond_{\{j\}}(q,\dot{q}) }.
}
\item The saltation matrix equation \eqref{eq:salt_general} can then be written as
\eqnn{
	\label{eq:salt_reduced_reset}
	\Xi_\w^\eta = 
		\prod_{\ell=0}^{\abs{K_a}+\abs{K_d}-1}\left(D\hat{R}_{\ell}(\rho,\rhodotm) + S_\ell\right).
}

\item By computing the acceleration of each limb using \eqref{eq:accel} and $\Diamond_{\eta(\ell)}$ has a block
	diagonal structure given by \eqref{eq:diagDiamond}, clearly
\eqnn{ 
	\label{eq:accel_diff}
	\alpha_{w(\ell+1)}(\rho, \rhodotp) - \Diamond_{\eta(\ell)}(\rho, \rhodotm )\alpha_{\w(\ell)}(\rho, \rhodotm)  \\
&\hspace{-50mm}=\begin{cases}
			\alpha_{\set{\eta(\ell)}}(\rho, \Delta_{\set{\eta(\ell)}}(\rho,\dot{\rho})\dot{\rho})-
			\Diamond_{\set{\eta(\ell)}}(\rho, \dot{\rho}) \alpha_\emptyset(\rho,\dot{\rho}) & \text{if $\w(\ell+1)=\w(\ell)\bigcup \set{\eta(\ell)}$}\\
			\alpha_\emptyset(\rho, \Delta_{\set{\eta(\ell)}}(\rho,\dot{\rho})\dot{\rho})-
			\Diamond_{\set{\eta(\ell)}}(\rho, \dot{\rho}) \alpha_\emptyset(\rho,\dot{\rho}) & \text{if $\w(\ell+1)=\w(\ell)\bigcup \set{\eta(\ell)}$}.
		
	\end{cases}
}

\item From \eqref{eq:DRI}, the vector field difference in the calcuation $S_\ell$ for only an activation reduces to 
	
	\eqnn{
		F_{(\ell+1)}&(\rho,\rhodotp)-DR_{\w(\ell)\bigcup\set{\eta(\ell)}}(\rho,\rhodotm)F_{\ell}(\rho,\rhodotm)  \\
		&=F_{\ell+1}(\rho,\rhodotp)-DR_{\ell}(\rho,\rhodotm)F_{\ell}(\rho,\rhodotm) \\
		&=	\left[ \mat{c}{\rhodotp[j] \\ \alpha_{\w(\ell+1)}(\rho,\rhodotp)[j]} - 
			\mat{c c}{ \I & 0 \\ \Box_{\set{\eta(\ell)}}^j(\rho,\rhodotm) & 
				\Diamond_{\set{\eta(\ell)}}^j(\rho, \rhodotm)} 
			\mat{c}{\rhodotm[j] \\ \alpha_{\w(\ell)}(\rho, \rhodotm)[j]} \right]_{j=0}^n ,}
where
\eqnn{
	\label{eq:simp_act_diff}
		&\mat{c}{\rhodotp[j] \\ \alpha_{\w(\ell+1)}(\rho,\dot{\rho}_\ell)[j]} - 
			\mat{c c}{ \I & 0 \\ \Box_{\set{\eta(\ell)}}^j(\rho,\rhodotm) & 
				\Diamond_{\set{\eta(\ell)}}^j(\rho, \rhodotm)} 
			\mat{c}{\rhodotm[j] \\ \alpha_{\w(\ell)}(\rho, \rhodotm)[j]} \\
	&=
    \begin{cases}
				\mat{c}{0\\0} & \parbox[t]{.2\columnwidth}{if $j\neq \eta(\ell)$  and\\ $j\neq 0$} \\ 
				\mat{c}{0\\ \alpha_{\w(\ell+1)}(\rho,\rhodotp)[0]-\alpha_{\w(\ell)}(\rho, \rhodotm)[0]} & \text{ if $j=0$} \\
				\mat{c}{\rhodotp[j] - \rhodotm[j] \\ \alpha_{\w(\ell+1)}(\rho, \rhodotp)[j]-\Box_{\set{\eta(\ell)}}^{j}(\rho,\dot{\rho})\rhodotm[j] -
					\Diamond_{\set{\eta(\ell)}}^{j}(\rho, \rhodotm)\alpha_{\w(\ell)}(\rho, \rhodotm)[j]}  & \text{ if $j=\eta(\ell)$}
	\end{cases}\\
	&=
	\begin{cases}
				\mat{c}{0\\0} & \parbox[t]{.2\columnwidth}{if $j\neq \eta(\ell)$  and\\ $j\neq 0$} \\
				\mat{c}{0\\ \alpha_{\set{\eta(\ell)}}\paren{\rho, \Delta_{\set{\eta(\ell)}}(\rho,\dot{\rho})\dot{\rho}}[0]-\alpha_{\emptyset}(\rho, \dot{\rho})[0]}& \text{ if $j=0$} \\
				\mat{c}{\paren{\Delta_{\set{\eta(\ell)}}(\rho,\dot{\rho})\dot{\rho}}[j] - \dot{\rho}[j] \\ 
					\alpha_{\set{\eta(\ell)}}\paren{\rho, \Delta_{\set{\eta(\ell)}}(\rho,\dot{\rho})\dot{\rho}}[j]-
			\Box_{\set{\eta(\ell)}}^{j}(\rho,\dot{\rho})\dot{\rho}[j]
			-\Diamond_{\set{\eta(\ell)}}^{j}(\rho, \dot{\rho})\alpha_\emptyset(\rho,\dot{\rho})[j]}  & \text{ if $j=\eta(\ell)$}.
	\end{cases} }
The equality in the last step for the $j=\eta(\ell)$ case can be seen from the block diagonal structure of $\Box_{\eta(\ell)}$ and \eqref{eq:accel_diff}.
Thus, for the case of only an activation,
\eqnn{
	\label{eq:salt_vecfield_diff}
		F_{(\ell+1)}(\rho,\rhodotp)-&DR_{\w(\ell)\bigcup\set{\eta(\ell)}}(\rho,\rhodotm)F_{\ell}(\rho,\rhodotm) 
		=F_{\set{\eta(\ell)}}(\rho,\dot{\rho}_{\eta(\ell)})-DR_{\set{\eta(\ell)}}(\rho,\dot{\rho}_\emptyset)F_{\emptyset}(\rho,\dot{\rho}_\emptyset).
	}
	Clearly, it can be shown by algebraic manipulation similar to \eqref{eq:simp_act_diff}, the equality 
\eqnn{
	\label{eq:salt_vecfield_diff_bounce}
		F_{\ell+1}(\rho,\rhodotp)-&DR_{\w(\ell)\bigcup\set{\eta(\ell)}}(\rho,\rhodotm)F_{\ell}(\rho,\rhodotm) 
		=F_{\emptyset}(\rho,\dot{\rho})-DR_{\set{\eta(\ell)}}(\rho,\dot{\rho}_\emptyset)F_{\emptyset}(\rho,\dot{\rho}_\emptyset)
	}

holds for the case of an activation instantly causing a deactivation.

\item The saltation matrix from \eqref{eq:salt_reduced_reset} can be further simplified using the 
	independence of the inner product in $S_\ell$ \eqref{eq:salt_scaling}, 
	along with the flow differences \eqref{eq:salt_vecfield_diff} and \eqref{eq:salt_vecfield_diff_bounce} to
\eqnn{ \label{eq:salt_reduced}
	\Xi_\w^\eta=
		\prod_{\ell=0}^{\abs{K_a}+\abs{K_d}-1}\left(DR_{\ell}(\rho,\dot{\rho}) + \widetilde{S}_\ell\right),
}
where 
\eqnn{
	\widetilde{S}_\ell = \frac{1}{Da_{\eta(\ell)}(\rho[\eta(\ell)])\dot{\rho}[\eta(\ell)]} 
	\paren{ F_{\w(\ell+1) \backslash \w(\ell)}(\rho, \dot{\rho}) - DR_\ell(\rho,\dot{\rho}) F_\emptyset(\rho,\dot{\rho})}
	Dh_{\eta(\ell)}(\rho, \dot{\rho}).
}

\item Given the constraints are only dependent upon position, clearly
\eqnn{\label{eq:salt_SDR}
		\widetilde{S}_jDR_{\{i\}}(\rho,\dot{\rho}) = \widetilde{S}_j }
for all $i,j\in \{0,\dots,\abs{K}-1\}$.
\item Next we show that 
\eqnn{\label{eq:salt_DRS}
		DR_{\{j\}}(\rho,\dot{\rho})\widetilde{S}_i&= \widetilde{S}_i
	}
for $j \neq i$. Given the block structure of the corresponding matrices, we make the following observations:
\begin{enumerate}
	\item only the columns associated with the indices for $q_i$ are nonzero in $\widetilde{S}_i$;
	\item only the rows associated with $q_0$ and $q_i$ are nonzero in $\widetilde{S}_i$;
	\item since $j\neq i$, $\Box^i_{\set{j}}(\rho,\dot{\rho})=0$ and $\Diamond_{\set{j}}^i(\rho, \dot{\rho})=I_{d_i}$.
\end{enumerate}
The nonzero elements of $S^i$ are thus multiplied by an identity like matrix.

\item Given \eqref{eq:salt_SDR} and \eqref{eq:salt_DRS}, the saltation matrix \eqref{eq:salt_reduced} 
	contains the matrix product $\widetilde{S}_\ell \widetilde{S}_k$, 
	where $\ell > k$ and $k,\ell\in \set{0, \dots, \abs{K}-1}$. 
	Within this matrix product, lies the inner product
	\eqnn{
		Dh_\ell(\rho,\dot{\rho})
		\paren{F_{\w(k+1) \backslash \w(k)}(\rho,\dot{\rho})-DR_{k}(\rho,\dot{\rho})F_{\emptyset}(\rho,\dot{\rho})} 
			&=Da_\ell(\rho[\ell])\dot{\rho}[\ell] - Da_\ell(\rho)\dot{\rho}[\ell]\\
			&=0.
	}
	Hence 
\eqnn{\label{eq:salt_crossterm}\widetilde{S}_\ell \widetilde{S}_g=0.}
\item Given \eqref{eq:salt_SDR},\eqref{eq:salt_DRS}, and \eqref{eq:salt_crossterm}, the saltation matrix \eqref{eq:salt_reduced} can be written as
\eqnn{
	\label{eq:salt_final}
	\Xi_\omega^\eta(\rho,\dot{\rho}) = 
		DR_{K_a}(\rho,\dot{\rho}) + \sum_{k\in K_a}\widetilde{S}_k.
}
\item The Bouligand-derivative \cite[Chpt.~3]{Scholtes2012} of $\phi_t(q,\dot{q})$ in direction $(v,\dot{v})\in TQ$ is given by
	\eqnn{
		\label{eq:b-deriv}
		D\phi_t(q,\dot{q}; v, \dot{v}) = D\phi_{t-s}(\rho, \dot{\rho}) \Xi_\omega^\eta(\rho,\dot{\rho}) D\phi_s(q,\dot{q})\begin{bmatrix} v \\ \dot{v} \end{bmatrix},
	}
	where $s\in \R$ is the time of simultaneous activation and/or deactivation of constraints, and $\eta,\omega$ are determined by $(v,\dot{v})$ \cite[Eq.~65]{BurdenSastry2016siads}.
\item As the saltation matrix calculation \eqref{eq:salt_final} is independent of the word $\w$,
	\eqref{eq:b-deriv} can be rewritten as
	\eqnn{
		D\phi_t(q,\dot{q}; v, \dot{v}) = D\phi_{t-s}(\rho, \dot{\rho}) \paren{DR_{K_a}(\rho,\dot{\rho}) + \sum_{k\in K_a}\widetilde{S}_k }D\phi_s(q,\dot{q})\begin{bmatrix} v \\ \dot{v} \end{bmatrix}.
	}
	Then $D\phi_t(q,\dot{q}; .)$ is a linear function and the flow is classically differentiable \cite[Chpt.~3]{Scholtes2012}.
\end{enumerate}

\end{proof}
}
{}

\iftoggle{ieee}{
\renewcommand*{\bibfont}{\small}
}
{
\pagebreak
\hypersetup{linkcolor=blue}
}
\printbibliography

\end{document}